%% file: main_arxiv.tex
\documentclass[10pt,twocolumn,letterpaper]{article}

\usepackage{cvpr}              

\usepackage{times}
\usepackage{epsfig}
\usepackage{graphicx}
\usepackage{amsmath}
\usepackage{amssymb}
\usepackage{booktabs}
\usepackage{algorithmic}
\usepackage{algorithm}
\usepackage{blindtext}
\usepackage{caption}
\usepackage{dsfont}
\usepackage{bbm}
\usepackage{bbding}
\usepackage[para]{footmisc}
\usepackage{tcolorbox}
\usepackage{stfloats}
\usepackage{balance}
\usepackage{amsthm}
\usepackage{tcolorbox}
\usepackage{textcomp} 
\usepackage{multirow}
\usepackage[table]{xcolor}  
\usepackage{pifont}
\usepackage{lipsum,multicol}

\ifx\theorem\undefined
\newtheorem{theorem}{Theorem}

\theoremstyle{definition}

\newtheorem{assumption}{Assumption}

\definecolor{cvprblue}{rgb}{0.21,0.49,0.74}
\usepackage[pagebackref,breaklinks,colorlinks,allcolors=cvprblue]{hyperref}

\def\paperID{5894} 
\def\confName{CVPR}
\def\confYear{2026}

\title{HiCoGen: Hierarchical Compositional Text-to-Image Generation \\ in Diffusion Models via Reinforcement Learning}

\def\spaces{}
\author{Hongji Yang\textsuperscript{1}, \spaces Yucheng Zhou\textsuperscript{1}, \spaces Wencheng Han\textsuperscript{1}, \\
\spaces Runzhou Tao\textsuperscript{2},  \spaces Zhongying Qiu\textsuperscript{2},  \spaces Jianfei Yang\textsuperscript{2},  \spaces Jianbing Shen\textsuperscript{1}\footnotemark[2]\\
\textsuperscript{1}{SKL-IOTSC, CIS, University of Macau} \\ 
\textsuperscript{2}{Zhejiang ZEEKR Automobile Research \& Development Co., Ltd.}\\
{\tt\small $\dagger$~Corresponding author}
}

\begin{document}
\maketitle
\input{sec/0_abstract}    
\input{sec/1_intro}

\input{sec/2_Related_Work}
\input{sec/3_Method}

\input{sec/4_Dataset}

\input{sec/5_Result}

\input{sec/6_Conclusion}
{
    \small
    \bibliographystyle{ieeenat_fullname}
    \bibliography{main}
}

\input{sec/X_suppl}

\end{document}

%% file: sec/0_abstract.tex
\vspace{-2mm}
\begin{abstract}
Recent advances in diffusion models have demonstrated impressive capability in generating high-quality images for simple prompts. However, when confronted with complex prompts involving multiple objects and hierarchical structures, existing models struggle to accurately follow instructions, leading to issues such as concept omission, confusion, and poor compositionality.
To address these limitations, we propose a Hierarchical Compositional Generative framework (\textbf{HiCoGen}) built upon a novel \textbf{Chain of Synthesis (CoS)} paradigm. Instead of monolithic generation, HiCoGen first leverages a Large Language Model (LLM) to decompose complex prompts into minimal semantic units. It then synthesizes these units iteratively, where the image generated in each step provides crucial visual context for the next, ensuring all textual concepts are faithfully constructed into the final scene.
To further optimize this process, we introduce a reinforcement learning (RL) framework. Crucially, we identify that the limited exploration of standard diffusion samplers hinders effective RL. We theoretically prove that sample diversity is maximized by concentrating stochasticity in the early generation stages and, based on this insight, propose a novel \textbf{Decaying Stochasticity Schedule} to enhance exploration. Our RL algorithm is then guided by a hierarchical reward mechanism that jointly evaluates the image at the global, subject, and relationship levels.
We also construct \textbf{HiCoPrompt}, a new text-to-image benchmark with hierarchical prompts for rigorous evaluation. Experiments show our approach significantly outperforms existing methods in both concept coverage and compositional accuracy.
\end{abstract}

%% file: sec/1_intro.tex
\vspace{-4mm}
\section{Introduction}
\vspace{-1mm}

\label{sec:intro}

\begin{figure}[htp]
    \centering
    \includegraphics[width=1\linewidth]{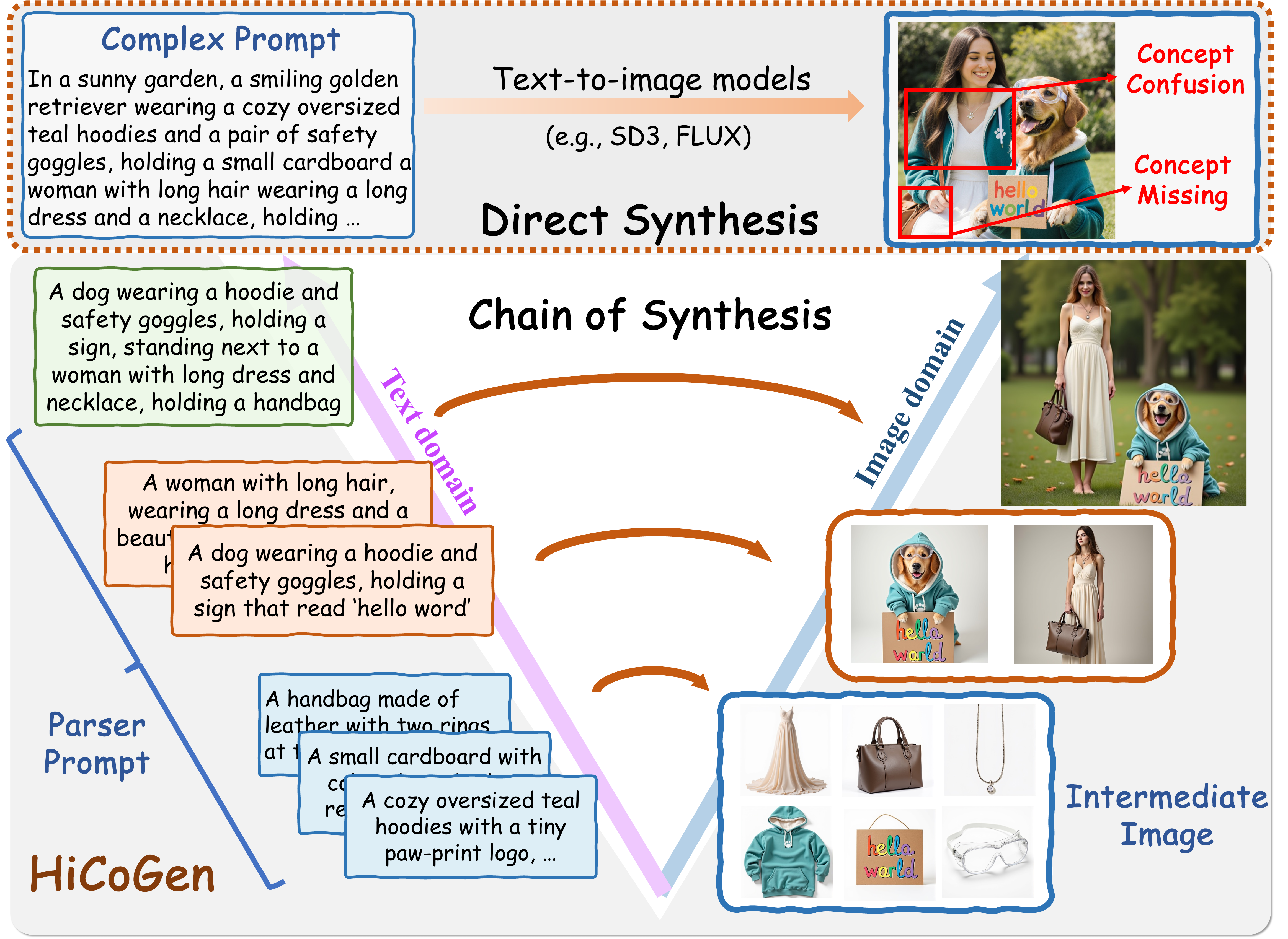}
    \vspace{-6mm}
    \caption{\textbf{The motivation of HiCoGen.} The semantic gap between text and images widens as the complexity of the text increases, particularly involving the prompts with a hierarchical relationship. While a single T2I model performs well in generating individual objects, it suffers from concept missing and confusion when processing complex prompts. HiCoGen employs a \textbf{Chain of Synthesis} for complex text to preserve the semantic content.}
    \vspace{-5mm}
    \label{fig:motivation}
\end{figure}

Diffusion models~\cite{ho2020denoising} have led to significant improvements in text-to-image (T2I) generation. Many T2I models~\cite{ saharia2022photorealistic, peebles2023scalable,rombach2022high,flux2024,flux1kreadev2025,esser2024scaling} have been proven to be able to produce high-quality and diverse images. 
However, existing image diffusion models fail to generate images that cover all concepts within long and complex prompts. Despite being equipped with more powerful text encoders (CLIP~\cite{radford2021learning}, T5-Encoder~\cite{raffel2020exploring}), the challenges of image generation with complex prompts remain unresolved.
Specifically, complex prompts contain extensive information, resulting in a large semantic gap between text and image domains, as shown in Fig.~\ref{fig:motivation}. This leads to several issues:
\ding{172}~\textbf{concept missing}, where the model overlooks parts of the prompt due to overly long textual descriptions;
\ding{173}~\textbf{concept confusion}, as the text encoder fails to capture hierarchical relationships, causing the model to rely on priors and generate incorrect interactions between objects;
\ding{174}~\textbf{overall quality decline}, where the visual fidelity of the generated image degrades.

In stark contrast, Large Language Models (LLMs) demonstrate powerful inference-time reasoning by scaling not only the amount of computation, but the \emph{reasoning process} itself: techniques such as Chain-of-Thought~\cite{wei2022chain} and reinforcement-style inference frameworks (e.g., Deepseek-R1~\cite{guo2025deepseek}) explicitly decompose a complex problem into a sequence of intermediate steps. This form of \emph{content-level} inference-time scaling stands apart from recent work on diffusion inference-time scaling~\cite{ma2025inference}, which allocates additional compute to the \emph{decoding process} (e.g., exploring more noises or sampling trajectories for a fixed holistic prompt) while keeping the textual input and its semantics unchanged. Inspired by the former line of work, we ask a different question: ``\emph{Instead of only scaling the sampling procedure of diffusion models, can we introduce a similar step-by-step, content-aware processing paradigm into visual synthesis itself to overcome the limitations of holistic generation?}''

To realize this principle in visual synthesis, we propose {\bf Chain of Synthesis (CoS)}, a new paradigm that replaces monolithic generation with a progressive, part-by-part construction process. Building on this paradigm, our framework, {\bf HiCoGen}, is the first system to realize this paradigm, engineered specifically for prompts with complex hierarchical and relational structures (e.g., ``\textit{one person holding an object standing next to another person who is also holding an object}''). Within HiCoGen, we first leverage an LLM to decompose the complex prompt into a set of minimal semantic units and rewrite the prompt for generation. 
The framework then synthesizes these units iteratively: the image generated in each step becomes the critical visual context for the next. This process constructs a visual compositional pathway, which progressively replaces entangled textual descriptions with concrete visual information, ensuring a faithful construction of the final complex scene.

Moreover, to optimize the fidelity of the CoS pipeline, we introduce a reinforcement learning (RL) framework based on GRPO \cite{guo2025deepseek}. However, effectively applying RL to diffusion models is a major challenge, as standard sampling methods offer a limited exploration space, hindering the discovery of optimal generation policies. To overcome this fundamental barrier, we first conduct a rigorous theoretical analysis of the stochastic sampling dynamics. We prove that sample diversity is critically dependent on early-stage stochasticity. Based on this insight, we introduce a \textbf{Decaying Stochasticity Schedule}, a novel sampling mechanism that provably maximizes exploration and is thus essential for enabling effective RL. Building on this enhanced sampler, our RL algorithm is guided by a novel \textbf{hierarchical reward mechanism}, which supplies global-, subject-, and relationship-level rewards for step-by-step supervision.


Finally, to facilitate rigorous and standardized evaluation of compositional reasoning in T2I models, a capability central to our work, we construct \textbf{HiCoPrompt}. It is a new, systematically generated benchmark featuring prompts with clearly defined hierarchical and compositional structures. Extensive experiments on HiCoPrompt validate that our HiCoGen significantly outperforms existing models.

Our contribution can be summarized as follows:
\begin{itemize}
\item We introduce a novel Chain of Synthesis (CoS) paradigm for compositional generation, and propose HiCoGen, the first framework to realize this step-by-step process for T2I models, overcoming limitations of monolithic methods.
\item Combined with our enhanced sampler and a GRPO-based algorithm, we design a novel hierarchical reward mechanism that provides fine-grained supervision on compositional aspects.
\item We identify the limited exploration of standard diffusion samplers as a key barrier to RL-based optimization. We prove that optimal sample diversity is achieved by concentrating stochasticity in the early generation stages and introduce a Decaying Stochasticity Schedule to realize it.
\item We construct the HiCoPrompt benchmark with explicit hierarchical structures to facilitate robust evaluation in compositional text-to-image generation.
\end{itemize}

%% file: sec/2_Related_Work.tex
\vspace{-1mm}
\section{Related Work}
\vspace{-1mm}
\label{Related Work}

\noindent\textbf{Text-to-Image Diffusion Models.}
From UNet-based diffusion models (e.g., LDM~\cite{rombach2022high}, SDXL~\cite{podell2023sdxl}) to the more powerful Transformer-based diffusion models (e.g., DiT~\cite{peebles2023scalable}, SD3~\cite{esser2024scaling}, FLUX~\cite{flux2024}), Text-to-image diffusion models have been proven to produce high-fidelity images.
However, they often fail to generate images under complex prompts involving multiple entities or subjects. To introduce more controllability into the diffusion model, recent work~\cite{zhang2023adding, ye2023ip, qin2023unicontrol, li2025controlnet, yang2025dc, ma2024hico} has included additional adapters or adjusted the prompt for a more accurate generation.

\noindent\textbf{Subject-driven Diffusion Models.}
While T2I models can generate diverse images of generic concepts, they still struggle to synthesize images of a specific subject provided by users. DreamBooth~\cite{ruiz2023dreambooth} and textual inversion~\cite{gal2023an} apply LoRA~\cite{hu2022lora} to insert a certain concept and perform a subject-driven generation. However, these methods only focus on the trained subject and are limited in their generalization. To improve generalization, recent works~\cite{wang2024ms, xiao2025omnigen, tan2025ominicontrol, wu2025less} have introduced the in-context generation framework for subject-driven generation or customized generation.
Nevertheless, In-Context Generation still suffers from missing the given concepts or confusing different concepts.

\noindent\textbf{Diffusion Reinforcement Learning.}
Inspired by the great success of LLMs finetuning using RL from Human Feedback (RLHF)~\cite{ouyang2022training}, DDPO~\cite{black2024training}, Diffusion-DPO~\cite{wallace2024diffusion} and DPOK~\cite{fan2023dpok} introduced Direct Preference Optimization (DPO)~\cite{rafailov2023direct} into diffusion models to align with human preferences. 
Recent works~\cite{yu2025genflowrl, guo2025can, black2023training, guo2025can} validate the improvement of the diffusion with reinforcement learning.
Motivated by the success of Deepseek-R1~\cite{guo2025deepseek} in training the thinking process with RL, Flow-GRPO~\cite{liu2025flow} and DanceGRPO~\cite{xue2025dancegrpo} employ GRPO~\cite{shao2024deepseekmath} to further enhance the performance of visual generation. 

%% file: sec/3_Method.tex
\vspace{-1mm}
\section{Method}
\vspace{-1mm}
\label{Method}

\begin{figure*}
    \centering
    \includegraphics[width=1.0\linewidth]{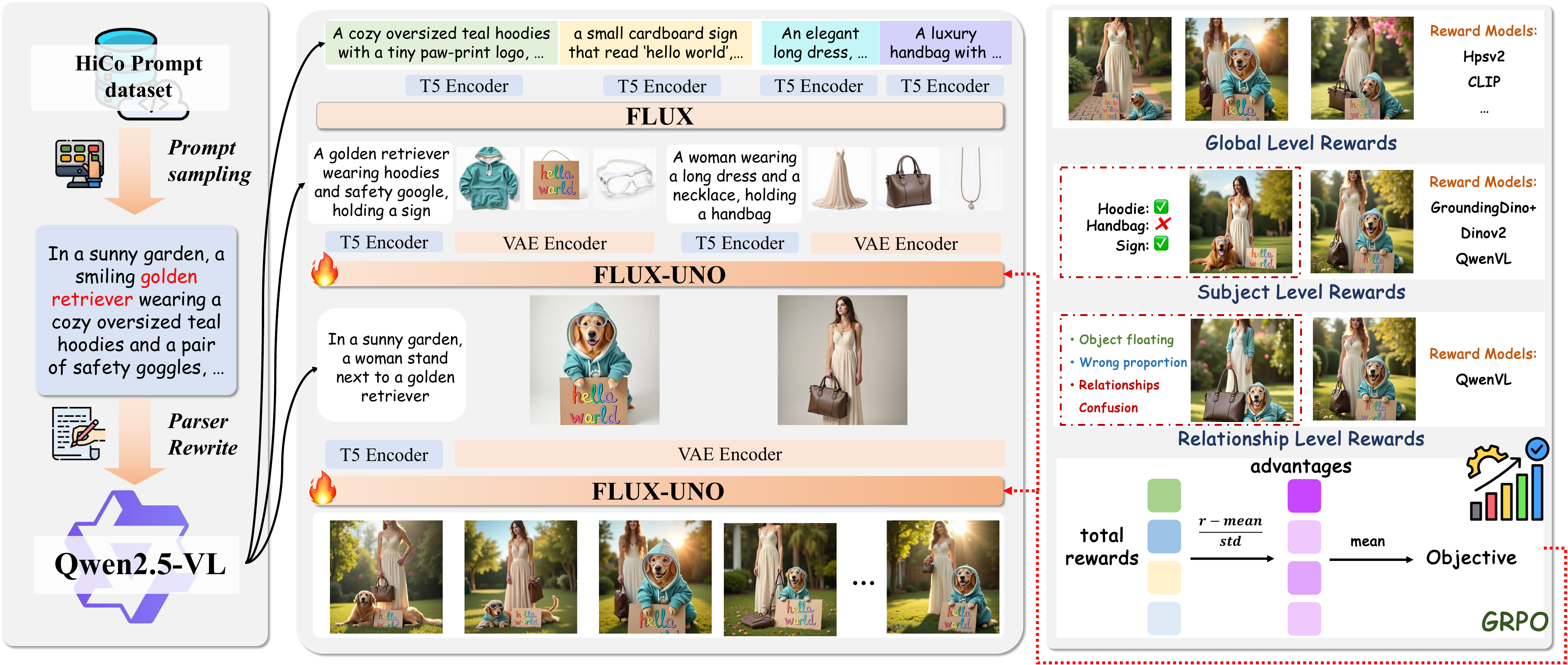}
    \vspace{-6mm}
    \caption{The overall pipeline of our proposed HiCoGen framework. When facing a complex hierarchical compositional prompt, HiCoGen applies the Chain of Synthesis to progressively construct the image part-by-part and employs in-context generative models to assemble the different components into the final image. This ensures all the concepts in the text domain are present in the image domain.}
    \vspace{-4mm}
    \label{fig:pipeline}
\end{figure*}

\subsection{Preliminary}
\vspace{-1mm}
\noindent\textbf{Diffusion Models.}
The forward process of diffusion is linearly adding a Gaussian noise item $\epsilon$ to the real data $z_{0}$:
\begin{equation}
    \text{\textbf{z}}_t=(1-t)\text{\textbf{z}}_0 + t\epsilon,
\end{equation}
where the $\text{\textbf{u}}=\epsilon-\text{\textbf{z}}_0$ is defined as the ``velocity'' in restrict flow~\cite{esser2024scaling}. To reach a lower noise level $\text{\textbf{s}}$, we have:
\begin{equation}
    \text{\textbf{z}}_s = \text{\textbf{z}}_t+\hat{\text{\textbf{u}}}\cdot (s-t)
\end{equation}
where the $\hat{\text{\textbf{u}}}$ is the denoising model output at timestep $t$. Given this output, the noisy latent can eventually be converted to a clear version.

\noindent\textbf{GRPO.}
Given a condition $c$, the generative models will sample multiple outputs $\{y_1, y_2,\dots,y_n\}$ from the model $\pi_{\theta_{old}}$ and compute a group of rewards $\{r_1,r_2,\dots,r_n\}$ with reward models. The advantage function $A_i$ for each output will be calculated as:
\begin{equation}
    A_i = \frac{r_i - \text{mean}(\{r_1, r_2, \dots, r_n\})}{\text{std}(\{r_1,r_2,\dots,r_n\})}
\end{equation}
Then, the policy model $\pi_\theta$ can be optimized by maximizing the following objective function:
\begin{equation}
\begin{aligned}
     & \mathcal{J}(\theta) = \mathbb{E}_{\{y_i\}^n_{i=1}\sim \pi_{\theta_{\text{old}}}(\cdot|\textbf{c}) , \text{\textbf{a}}_{t,i}\sim\pi_{\theta_{\text{old}}}{(\cdot|\text{\textbf{s}}_{t,i})}} \\
     & \left[\frac{1}{n}\sum_{i=1}^{n}\frac{1}{T}\sum_{t=1}^{T}\min\left(\rho_{t,i}A_i, \text{clip}(\rho_{t,i}, 1-\epsilon,1+\epsilon)A_i\right)\right],
\end{aligned}
\end{equation}

\vspace{-1mm}
\subsection{HiCoGen Framework}
\vspace{-1mm}
Diffusion models often struggle to follow long or complex text conditions, particularly when a prompt involves multiple complex subjects. 
The core idea of the HiCoGen Framework is to transform the intractable monolithic generative task into a chain of tractable single-subject generation and in-context composition tasks. 
Instead of demanding a single diffusion model to resolve complex compositional bindings in a single step, HiCoGen first generates a high-fidelity instance of a semantic unit and then iteratively integrates these results into the new scene, one by one, using the previously generated image as the representation of certain textual information, as shown in Fig.~\ref{fig:pipeline}.

\noindent\textbf{Parse\&Rewrite LLM.}
Given the complex input prompt $\mathcal{O}$, we first employ a semantic parser (e.g., LLMs) to decompose it into an ordered set of subject-centric components.
\begin{equation}
\begin{aligned}
    &\mathcal{O} = \{P^{(1)},P^{(2)},\cdots, P^{(n)}\}, \\
    &P^{(i)} = \{c^{(1)}, c^{(2)}, \cdots,c^{(m)}\},
\end{aligned}
\end{equation}
where $P_i$ is the sub-prompt for the $i$-th subject, and could be decomposed into multiple fine-grained component (or attributes) $c_j$. Then, LLM also rewrites the prompt at each level to make it feasible to generate an image. 

Each layer’s prompt provides detailed descriptions only for the specific subject or attribute at that level, while the higher layer includes only general categorical information without such details.

 
       

\noindent\textbf{Chain of Synthesis.}
To perform a CoS, a primary challenge is how to assemble the generated contents. For a text-to-image DiT model with multi-modal attention, the input $z$ is the concatenated result of the text token $c$ and the noisy latent $z_t$, and can be expressed as follows:
\begin{equation}
    z = \text{concat}(c, z_t),
\end{equation}
After generating the specific contents $\{z^0_0, z^1_0, \cdots z^m_0\}$ in the latent domain, we leave these generated contents to continue engaging in text-to-image generation. This can be formatted as follows:
\begin{equation}
    \hat{z} = \text{concat}(P^{(i)}, \hat{z_t}, z^0_0, z^1_0, \cdots z^m_0),
\end{equation}
Here, we create an input $\hat{z}$ with hierarchical information $\{z^0_0, z^1_0, \cdots z^m_0\}$ with text $P^{i}$, reducing the need for the original prompt description on specific contents. And this process can be continued until we generate an image that covers the entire prompt.
\begin{equation}
    z^{'} = \text{concat}(\mathcal{O}, z^{'}_t, \hat{z}^0_0, \hat{z}^1_0, \cdots \hat{z}^n_0),
\end{equation}
where $z_0$ denotes the clear latent at timestep $t=0$. 
This chained synthesis assembles all the generated results and also facilitates the hierarchical relationship generation.

\subsection{Hierarchical Reward}
\vspace{-1mm}
To provide a comprehensive and fine-grained reward signal for the quality and accuracy of diffusion models, we propose a novel hierarchical reward mechanism. It not only focuses on global text-to-image alignment or aesthetic quality but also aims at the image's internals to quantify the accuracy of key subjects and the relationships between them.
Our overall reward function, $R_\text{total}$, is a weighted sum of three core components: a global reward $R_\text{global}$, a subject-specific reward $R_\text{subject}$, and a subject relationship reward $R_\text{relationship}$.
\begin{equation}
    R_\text{total} = R_\text{global} + R_\text{subject} +  R_\text{relationship},
\end{equation}

\noindent\textbf{Global Reward.}
The global reward $R_\text{total}$ aims to assess the overall quality of the generated images from a holistic perspective. It consists of two main aspects: text-image alignment and human preference.
\begin{equation}
    R_{\text{global}}=w_\text{clip} \cdot S_\text{clip} + w_{hps} \cdot S_\text{hps}
\end{equation}
where $S_\text{clip}$ and$S_\text{hps}$ denote the image reward score from the CLIP model~\cite{radford2021learning} and human-preference score~\cite{wu2023better}. 

\noindent\textbf{Subject-Specific Reward.}
In many generation tasks, accurately depicting one or more key subjects is crucial. $R_{\text{subject}}$ focuses on evaluating the fidelity and attributes of these specific subjects.
For a specific subject, we first use the GroundingDino~\cite{liu2024grounding} to locate the positions of key subjects in the prompt, obtaining their bounding boxes, and then crop the subject.
After the corresponding subject is located, we assess the similarity between the cropped part and the previously generated intermediate image using the DINOv2~\cite{oquab2024dinov}. We compute the embedding vector using cosine similarity $\cos(\cdot, \cdot)$ and the process can be expressed as:
\begin{equation}
    S_\text{DINOv2}=\cos\left(\text{DINOv2}(I_\text{cropped}), \text{DINOv2}(I_\text{ref})\right)
\end{equation}
where $I_\text{cropped}$ and $I_\text{ref}$ denote the cropped image and the reference image, respectively.

In addition, many key attributes (such as color and pose) of the subject are difficult to capture solely through embedding similarity. For this, we utilize the powerful Vision-Language Models (VLM) as the rule-based reward model to perform the rewarding process. 
For the $N$ key subjects in the image, the subject reward $R_\text{subject}$ is an aggregation of the scores for each subject:
\begin{equation}
    R_{\text{subject}}=\frac{1}{N}\sum_{i=1}^N\left(w_\text{dino} \cdot S^{(i)}_{\text{DINOv2}} + w_{\text{vlm}} \cdot S^{(i)}_{\text{vlm}}\right)
\end{equation}
\noindent\textbf{Relationship Reward.}
Since the subject image is generated using an independent reference image, the following issues arise: 1) The rendered object appears to float or feel detached, with the interaction area appearing very stiff, resulting in an obvious pasted-on effect; 2) The object's proportions are significantly enlarged.

Therefore, the rewards focus on the relationship between the subject injection and the main content of the image, as well as whether the interaction is reasonable. We prompt VLM to focus on relationships between the given subjects. 
\begin{equation}
    R_{\text{relationship}}=\frac{1}{N}\sum_{i=1}^N\left(S^{(i)}_{\text{vlm}}\right)
\end{equation}
In the practical implementation, we make the VLM to verify the hierarchical relationships and the relative proportions.
Please refer to the supplementary for the prompts format.

\subsection{Decay Stochasticity Schedule}
\vspace{-1mm}
\noindent{\textbf{Sampling SDEs.}}
To enable exploration for reinforcement learning, we augment the standard reverse diffusion process with a controllable stochasticity term $\eta(t) \ge 0$. The resulting reverse-time SDE, evolving from $t=T$ to $0$, is:
\begin{equation}
    \!\!\!\text{d}\mathbf{z}_t \!\!=\!\! \left[\mathbf{f}(\mathbf{z}_t, t) \!-\! g(t)^2 \nabla_{\mathbf{z}_t}\log p_t(\mathbf{z}_t)\right]\text{d}t \!+\! g(t)\eta(t)\text{d}\mathbf{w}_t,\!\!
    \label{eq:reverse_sde}
\end{equation}
where $\mathbf{f}$ and $g$ are the drift and diffusion coefficients, and $\mathbf{w}_t$ is a standard Wiener process. Our objective is to determine the optimal schedule for $\eta(t)$ to maximize sample diversity.

\noindent{\textbf{Optimal Stochasticity Scheduling.}}
The denoising process has distinct temporal phases. Recent work~\cite{biroli2024dynamical} categorizes it into an early ``generalization time'' for establishing global structure and a later ``collapse time'' for refining details. This motivates concentrating exploration efforts early. We formalize this intuition by aiming to maximize sample diversity, quantified by $\text{Tr}(\text{Cov}(\mathbf{z}_0))$, under a fixed total stochasticity budget.

\begin{theorem}[Optimal Stochasticity Allocation for Diversity Maximization]
    \label{thm:optimal_schedule}
    Consider the reverse SDE in Eq.~\ref{eq:reverse_sde} with a fixed budget $\int_0^T \eta(t)^2 dt = C > 0$. Under Assumptions 1 and 2 below, the schedule $\eta(t)$ that maximizes the final sample diversity $\text{Tr}(\text{Cov}(\mathbf{z}_0))$ is a monotonically decreasing function of time $t$ (from $T$ to $0$).
\end{theorem}

\begin{assumption}[Linearizability]
\label{ass:linear}
The SDE dynamics can be analyzed by linearizing around a mean deterministic trajectory $\bar{\mathbf{z}}_t$.
\end{assumption}

\begin{assumption}[Increasing Contractivity]
\label{ass:contract}
The Jacobian of the drift, $\mathbf{A}_t$, becomes increasingly contractive as $t \to 0$. This reflects the increasingly dominant role of the score function in collapsing the latents onto the data manifold, transitioning from ``generalization time'' to ``collapse time''.
\end{assumption}

\begin{proof}[Proof of Theorem~\ref{thm:optimal_schedule}]
    Under Assumption \ref{ass:linear}, the covariance matrix $\mathbf{\Sigma}_t = \text{Cov}(\mathbf{z}_t)$ evolves according to the Lyapunov equation: $\frac{d\mathbf{\Sigma}_t}{dt} = \mathbf{A}_t \mathbf{\Sigma}_t + \mathbf{\Sigma}_t \mathbf{A}_t^\top + g(t)^2 \eta(t)^2 \mathbf{I}$.
    The final covariance, starting from $\mathbf{\Sigma}_T = \mathbf{0}$, is the solution integrated backwards in time:
    \begin{equation}
        \mathbf{\Sigma}_0 = \int_0^T \mathbf{\Phi}(0, s) \left( g(s)^2 \eta(s)^2 \mathbf{I} \right) \mathbf{\Phi}(0, s)^\top ds,
    \end{equation}
    where $\mathbf{\Phi}(t, s)$ is the state transition matrix for the linearized system. We thus solve the variational problem:
    \begin{equation}
        \max_{\eta(t)^2 \ge 0} \int_0^T \eta(s)^2 W(s) ds, \quad \text{s.t.} \int_0^T \eta(s)^2 ds = C,
    \end{equation}
    where the weight $W(s) = \text{Tr}\left( g(s)^2 \mathbf{\Phi}(0, s) \mathbf{\Phi}(0, s)^\top \right)$. The optimal solution allocates the budget proportionally to the weight $W(s)$.
    
    The weight $W(s)$ quantifies the impact of a perturbation at time $s$ on the final variance. Per Assumption \ref{ass:contract}, the system's contractivity increases as time progresses to $0$. This means the dynamics suppress perturbations more aggressively at later times. Consequently, the propagator $\mathbf{\Phi}(0, s)$ experiences less total suppression for larger $s$ (earlier times), implying $W(s)$ is a monotonically decreasing function as $s$ goes from $T$ to $0$.
    To maximize the objective, the budget $\eta(s)^2$ must be allocated where $W(s)$ is highest, i.e., at earlier times. Therefore, the optimal schedule $\eta(t)$ is monotonically decreasing.
\end{proof}
In our framework, we instantiate this principle on Rectified Flow. We augment its deterministic ODE with our principled stochastic schedule, implemented as a cosine decay:
\begin{align}
    \eta(t) &=\! \eta_{\min} \!+\! \frac{1}{2}(\eta_{\max}\!-\!\eta_{\min})\!\left(1 \!+\! \cos\!{\left(\!\frac{\pi (T_{\max}\!-\!t)}{T_{\max}}\!\right)}\!\!\right),\notag\\
    &t \in [0, T_{\max}].
\end{align}
This schedule focuses exploration on the crucial early stages of structural formation while preserving fidelity during the final refinement phase, thereby providing an optimal solution to the exploration-exploitation trade-off.

%% file: sec/4_Dataset.tex
\vspace{-1mm}
\section{HiCoPrompt Dataset}
\vspace{-1mm}
\label{Dataset}

\begin{figure}
    \centering
    \vspace{-2mm}
    \includegraphics[width=1.0\linewidth]{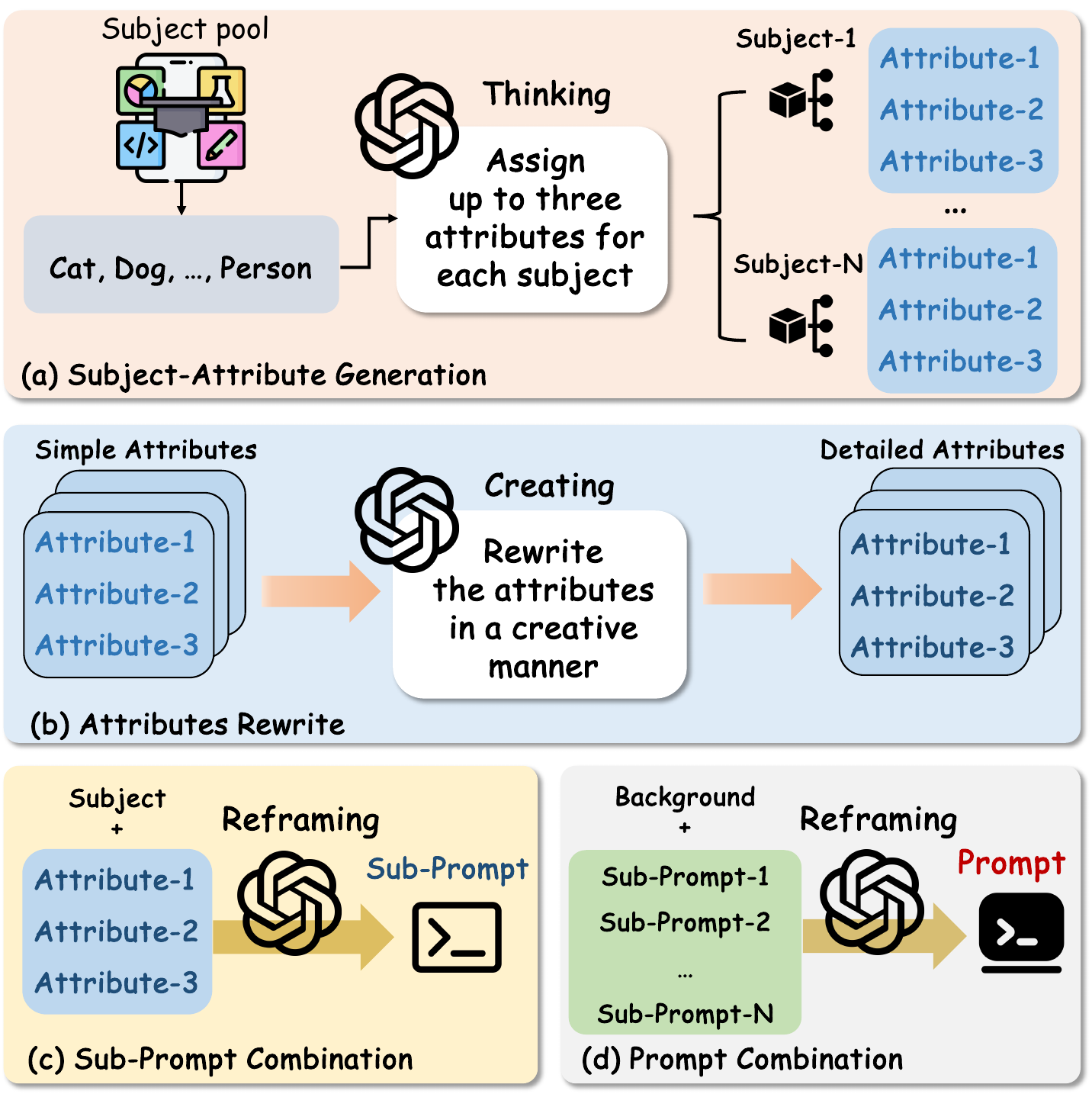}
    \vspace{-7mm}
    \caption{Illustration of our proposed HiCoPrompt dataset. It features clearly defined multiple hierarchical relationships, as well as thoroughly described concrete subjects and attributes.}
    \vspace{-5mm}
    \label{fig:hico_dataset}
\end{figure}

To construct a hierarchical compositional prompt dataset for challenging text-to-image generation, we propose the HiCoPrompt dataset. 
This dataset defines multiple subjects with multiple attributes for a single image, which is absent in other existing T2I benchmarks. Specifically, each prompt in the HiCoPrompt describes multiple subjects with structured specific attributes (e.g., clothing or holding). Each attribute is modified by descriptive qualifiers. The comparison with other T2I benchmarks is present in Tab.~\ref{tab:dataset}.
The benchmark comprises 3k prompts for testing and 12k prompts for RL training. The dataset will be available.
 
\noindent{\textbf{Subject Generation.}}
The initial phase focuses on establishing a foundation of core elements. We first select a diverse set of Primary Subjects (e.g., cat, king, astronaut) chosen for their high potential for visual representation and modification. For each subject, we employ ChatGPT~\cite{chatgpt_communication} to generate essential quantifiable core attributes. These attributes serve as standard slots for injecting descriptive details, thereby ensuring their richness and uniqueness.

\noindent{\textbf{Attribute Rewriting.}}
The second phase introduces descriptive variation by applying different modifiers and detailed descriptions. We ask the LLM to rewrite the attribute in an imaginative form. 
The rewriting transforms a simple category into a specific object with complex concepts. 

\noindent{\textbf{Sub-Prompt Reframing.}}
The full composite prompt is assembled by combining multiple subjects, each with a structured mix of these attributes. We require the LLM to construct the final prompt by incorporating the necessary prepositions or actions between subjects to integrate the distinct sentences.
The construction is present in Fig.~\ref{fig:hico_dataset}


%% file: sec/5_Result.tex
\vspace{-1mm}
\section{Experiment}
\vspace{-1mm}
\label{Result}

\begin{table*}
\setlength{\tabcolsep}{3pt}
    \vspace{-3mm}
    \centering
    \resizebox{\textwidth}{!}{
    \setlength{\tabcolsep}{6pt}
    \begin{tabular}{lccccccc}
    \toprule
    \textbf{Benchmark} & \textbf{Prompts} & \textbf{Sources}  & \textbf{\#Subjects/Prompt} & \textbf{Subject Description} & \textbf{Hierarchy} & \textbf{Composition} & \textbf{Length} \\ 
    \midrule
      DrawBench~\cite{otani2023toward}& ~~~200 & Human & 1-4 & Simple & $\times$ & $\times$ & ~~15 \\
      T2I-CompBench~\cite{huang2023t2i}  & 5,000 & Template & 1-4 & Simple & $\times$ & $\times$ & ~~10 \\ 
      T2I-CompBench++~\cite{huang2025t2i}  & 8,000 & Template & 1-4 & Simple & $\times$ & $\times$ & ~~14 \\
      DPG-Bench~\cite{hu2024ella} & 1,065 & Template & 1-3 & Simple & $\times$ & $\times$ & ~~84 \\
      LongBench-T2I~\cite{zhou2025draw}  & ~~~500 & LLM & 9 & Detailed & $\times$ & $\checkmark$ & 683 \\
      \rowcolor{gray!15}\textbf{HiCoPrompt(Ours)} & 3,000 & LLM & 4-12 & Detailed & $\checkmark$ & $\checkmark$ & 219 \\
    \bottomrule
    \end{tabular}}
    \vspace{-3mm}
    \caption{Comparison of other T2I benchmarks. The primary challenge of HiCoPrompt compared to other datasets lies in its detailed and specific descriptions for each subject, ensuring that even subjects within the same category exhibit distinct characteristics. At the same time, there exists a clear hierarchical relationship between subjects.}
    \label{tab:dataset}
    \vspace{-1mm}
\end{table*}

\subsection{Experimental Setup}
\vspace{-1mm}
We employ our experiment with flow-based diffusion models FLUX and the UNO LoRA weights for subject-driven generation. All experiments are conducted with 8 A100 80G GPUs. We employ the AdamW optimizer with a fixed learning rate of 1e-4 for fine-tuning. The rank and alpha of the LoRA are both set to 512, and the number of samples is set to 16. We employ Qwen2.5-VL-3B as both the parse and rewrite LLM and the reward model. The $\eta_{max}$ is set to 1.0 in our decaying stochasticity scheduler.
Please refer to the supplementary for more details.

\begin{table*}\small
    \setlength{\tabcolsep}{10pt}
    \centering
    \resizebox{\textwidth}{!}{
    \setlength{\tabcolsep}{15pt}
    \begin{tabular}{lccccccc}
    \toprule
    \textbf{Methods} & \textbf{$\text{Acc}_{\text{exist}}\uparrow$}  & \textbf{$\text{Acc}_{\text{attribute}}\uparrow$} & \textbf{$\text{Acc}_{\text{relationship}}\uparrow$} & \textbf{CLIP Score}$\uparrow$ & \textbf{HPSv2}$\uparrow$ \\ \midrule
    SDXL~\cite{podell2023sdxl} & 0.2672 & 0.0825 & 0.0760 & 0.2676 & 0.2379 \\
    Stable-Diffusion-3~\cite{esser2024scaling} & 0.3669 & 0.2991 & 0.3026 & 0.2774 & 0.2810  \\
    FLUX.1-dev~\cite{flux2024} & 0.4456 & 0.3805 & 0.4035 & 0.2628 & 0.2974 \\
    FLUX.1-Krea~\cite{flux1kreadev2025} & 0.5172 & 0.5992 & 0.6400 & 0.2641 & 0.2952 \\
    Qwen-Image~\cite{wu2025qwenimagetechnicalreport} & 0.6292 & 0.6907 & 0.7829 & 0.2687 & 0.2913 \\
    \midrule
    MS-Diffusion~\cite{wang2024ms} & 0.4552 & 0.1274 & 0.1960 & 0.2556 & 0.2330 \\
    OminiGen~\cite{xiao2025omnigen} & 0.5818 & 0.6791 & 0.6789 & 0.2646 & 0.2893 \\
    OminiControl~\cite{tan2025ominicontrol} & 0.4032 &0.6583 & 0.3044 & 0.2607 & 0.2886 \\
    UNO~\cite{wu2025less} & 0.4400 & 0.6923 & 0.3697 & 0.2691 & 0.2698 \\
    \midrule
    \rowcolor{gray!15}\textbf{HiCoGen(ours)} & \textbf{0.7127} & \textbf{0.7673} & \textbf{0.8203} & \textbf{0.3192} & \textbf{0.3357} \\

    \bottomrule
    \end{tabular}}
    \vspace{-3mm}
    \caption{Comparison of the proposed HiCoGen and other text-to-image models. HiCoGen outperforms other text-to-image models or subject-driven generative models in all metrics.}
    \label{tab:main_exp}
    \vspace{-1mm}
\end{table*}

\subsection{Evaluation Metrics}
\vspace{-1mm}
To further investigate how the RL fine-tuning improves the proposed HiCoGen, we design three corresponding metrics to evaluate the results based on GPT-4o~\cite{chatgpt_communication}. We prompt the GPT-4o to perform reasoning on three dimensions: 1) Does the subject/attribute exist? 2) Does the attribute align with the given description in detail? 3) Are the subject's relationships (e.g., interaction, proportion) accurate? 
In addition, we employ the commonly used CLIP score~\cite{radford2021learning} and HPSv2~\cite{wu2023better} score for comparison.

\subsection{Quantitative Results}
\vspace{-1mm}

We primarily compare our HiCoGen with two categories of methods: end-to-end text-to-image generative models (e.g., SDXL~\cite{podell2023sdxl}, SD3~\cite{esser2024scaling}, FLUX~\cite{flux2024}) and subject-driven text-to-image generative models (e.g., MS-Diffusion~\cite{wang2024ms}, UNO~\cite{wu2025less}). For these subject-driven models that only support single-subject generation, we employed a multi-step approach to inject the new subject with the prompt into the image one by one.
As shown in Tab.~\ref{tab:main_exp}, HiCoGen outperforms other text-to-image models in all metrics. For subjects existing and attributes accuracy, HiCoGen achieves a 9\% improvement compared to Qwen-Image, which is known for handling complex prompts.
HiCoGen also surpasses FLUX and Qwen-Image in the visual appeal metric HPSv2 by 0.04.
It is worth noting that using a subject-driven generative model (e.g., UNO) in such tasks preserves attribute details in the prompt better than other text-to-image models, although concept missing still occurs.

\begin{figure*}
    \centering
    \includegraphics[width=1\linewidth]{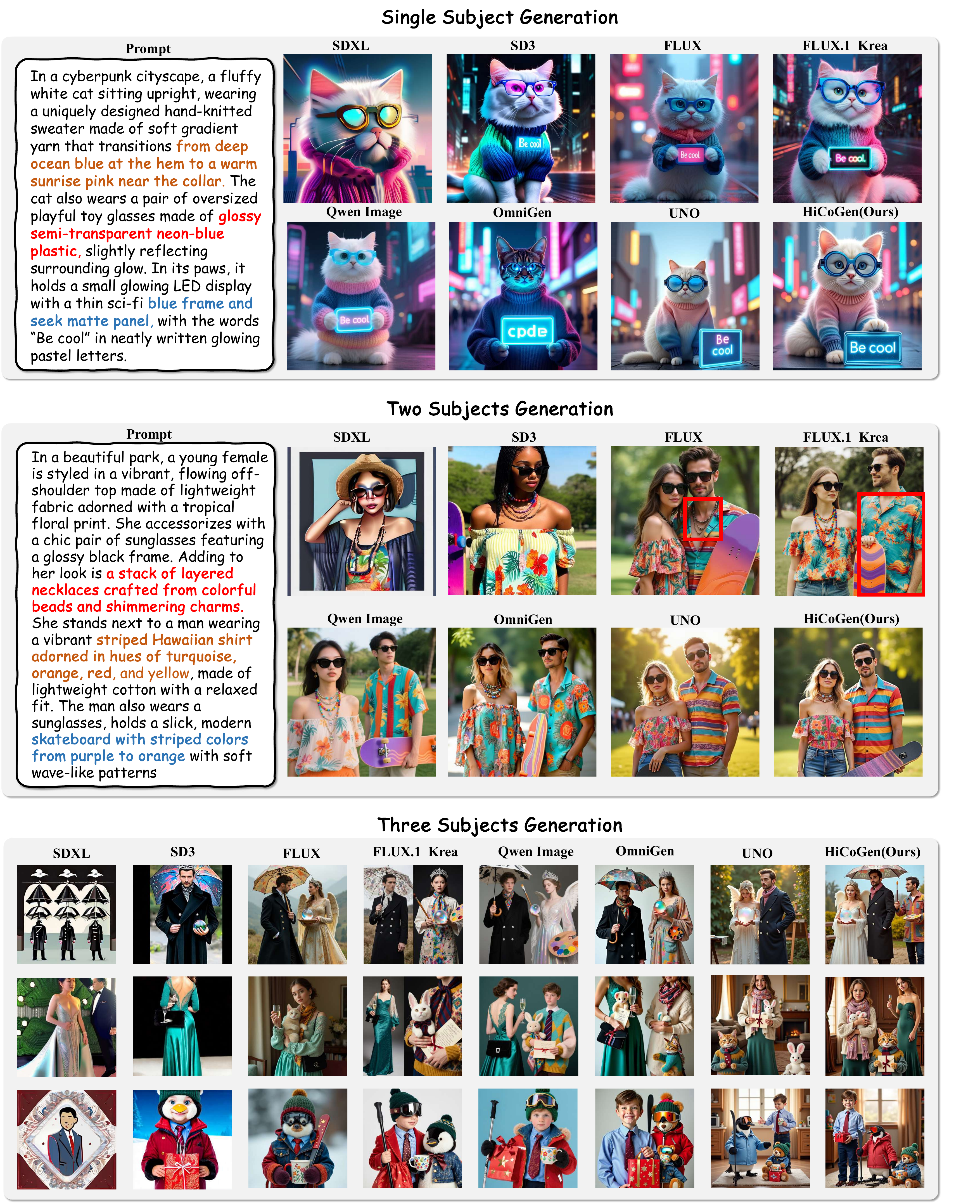}
    \vspace{-3mm}
    \caption{The visual result of our HiCoGen. HiCoGen greatly mitigates the issue of concept missing or confusion in image generation. When handling prompts with clear hierarchical relationships and multiple complex subjects, it significantly outperforms other T2I models.}
    \label{fig:visual_result}
\end{figure*}


\begin{table*}\small
    \vspace{-1mm}
    \setlength{\tabcolsep}{20pt}
    \centering
    \resizebox{\textwidth}{!}{
    \setlength{\tabcolsep}{10pt}
    \begin{tabular}{ccc|ccccc}
    \toprule
    \multicolumn{3}{c|}{\textbf{Reward}} & \multirow{2}{*}{\textbf{$\text{Acc}_{\text{exist}}\uparrow$}}  & \multirow{ 2}{*}{\textbf{$\text{Acc}_{\text{attribute}}\uparrow$}} & \multirow{ 2}{*}{\textbf{$\text{Acc}_{\text{relationship}}\uparrow$}} & \multirow{ 2}{*}{\textbf{CLIP Score}$\uparrow$} & \multirow{2}{*}{\textbf{HPSv2}$\uparrow$} \\ \cmidrule{1-3}
    Global-Level & Subject-Level & Relationship-Level  &  &  &  &  & \\ \midrule
    $\checkmark$  & $\checkmark$ & $\checkmark$  & \textbf{0.7128} & 0.7672 & \textbf{0.8202} & 0.3192 & \textbf{0.3357}  \\
    & $\checkmark$  & $\checkmark$ & 0.6828 & \textbf{0.7743} & 0.7810 & 0.2670 & 0.2720 \\
    $\checkmark$  &  & $\checkmark$ & 0.6924 & 0.6206 & 0.7164 & \textbf{0.3208} & 0.3346 \\
     $\checkmark$ & $\checkmark$  &  & 0.7038& 0.7558 & 0.7689 & 0.3106 & 0.3327 \\ 
      \bottomrule
    \end{tabular}}
    \vspace{-3mm}
    \caption{The ablation studies of using different rewards in HiCoGen.}
    \vspace{-4mm}
    \label{tab:ablation_reward}
\end{table*}

\subsection{Qualitative Results}
\vspace{-1mm}

In Fig.~\ref{fig:visual_result}, we present visual results comparing our HiCoGen with other text-to-image methods. 
The conventional T2I model is good at depicting a single subject. Both FLUX and FLUX-Krea achieve good visual results in the HiCoPrompt. However, there are some inconsistencies between the image and the given prompt. For example, they fail to realize the description ``\textit{from deep ocean blue at the hem to a warm sunrise pink near the collar}'' and FLUX-Krea incorrectly adds an extra paw of the cat to ``\textit{hold the LED board.}''
When dealing with multiple subjects, text-to-image models tend to produce images with conceptual confusion. For example, FLUX-Krea~\cite{flux1kreadev2025} and Qwen-Image~\cite{wu2025qwenimagetechnicalreport} incorrectly apply the ``\textit{floral clothing}'' attribute to a different subject, or FLUX assigns the ``\textit{stacked necklaces}'' to the wrong subject. 


\subsection{Ablation Study}
\vspace{-1mm}


\begin{table}\footnotesize
    \centering
    \setlength{\tabcolsep}{3pt}
    \resizebox{\linewidth}{!}{
    \begin{tabular}{l|ccc|ccc|ccc}
    \toprule
        \multirow{2}{*}{\textbf{Models}} & \multicolumn{3}{c}{\textbf{$\text{Acc}_{\text{exist}}\uparrow$}}  & \multicolumn{3}{c}{\textbf{$\text{Acc}_{\text{attribute}}\uparrow$}}&\multicolumn{3}{c}{\textbf{$\text{Acc}_{\text{relationship}}\uparrow$}} \\ \cmidrule{2-10}
        ~ & 1 & 2 & 3 & 1 & 2 & 3 & 1 & 2 & 3 \\ \midrule 
        Qwen-Image & \textbf{0.93} & 0.80 & 0.36 & 0.93 &  0.61 & 0.54 & \textbf{0.89}  & 0.85  &  0.61 \\
        UNO & 0.62 & 0.42  & 0.28 & 0.94 & 0.62  & 0.53  & 0.52  & 0.34  & 0.25  \\ \midrule
        \rowcolor{gray!15}\textbf{HiCoGen} & 0.92 & \textbf{0.81} & \textbf{0.41} & \textbf{0.95} & \textbf{0.74} & \textbf{0.62} & \textbf{0.89} & \textbf{0.86} & \textbf{0.71} \\
    \bottomrule
    \end{tabular}}
    \vspace{-3mm}
    \caption{Different number of subjects in generated images.}
    \vspace{-5mm}
    \label{tab:num_subject}
\end{table}

For the ablation studies, we conduct experiments on our HiCoGen to analyze the contributions of different reward components. As shown in Tab.~\ref{tab:ablation_reward}, the results show that combining global, subject, and relationship rewards yields the best overall performance. Removing the subject reward significantly decreases attribute accuracy (about 14\% decline), while removing the relationship reward weakens relationship construction (about 4\% reduction in accuracy).

\vspace{-1mm}
\subsection{Discussion}
\vspace{-1mm}

\noindent\textbf{The accuracy of the generated subject.}
In qualitative results, we have preliminarily demonstrated that the more subjects involved within the prompt, the more challenging for the generative model to generate the corresponding images. When focused on generating 1-2 subjects, the quality and accuracy of the generated images are both excellent. However, when the number of subjects increases to 3, the accuracy shows a significant decline (51\% in existing accuracy and 33.5\% in attributes). Despite this, HiCoGen still outperforms Qwen-Image and UNO, as shown in Tab~\ref{tab:num_subject}. 

\begin{figure}
    \centering
    \vspace{-1mm}
    \includegraphics[width=1.0\linewidth]{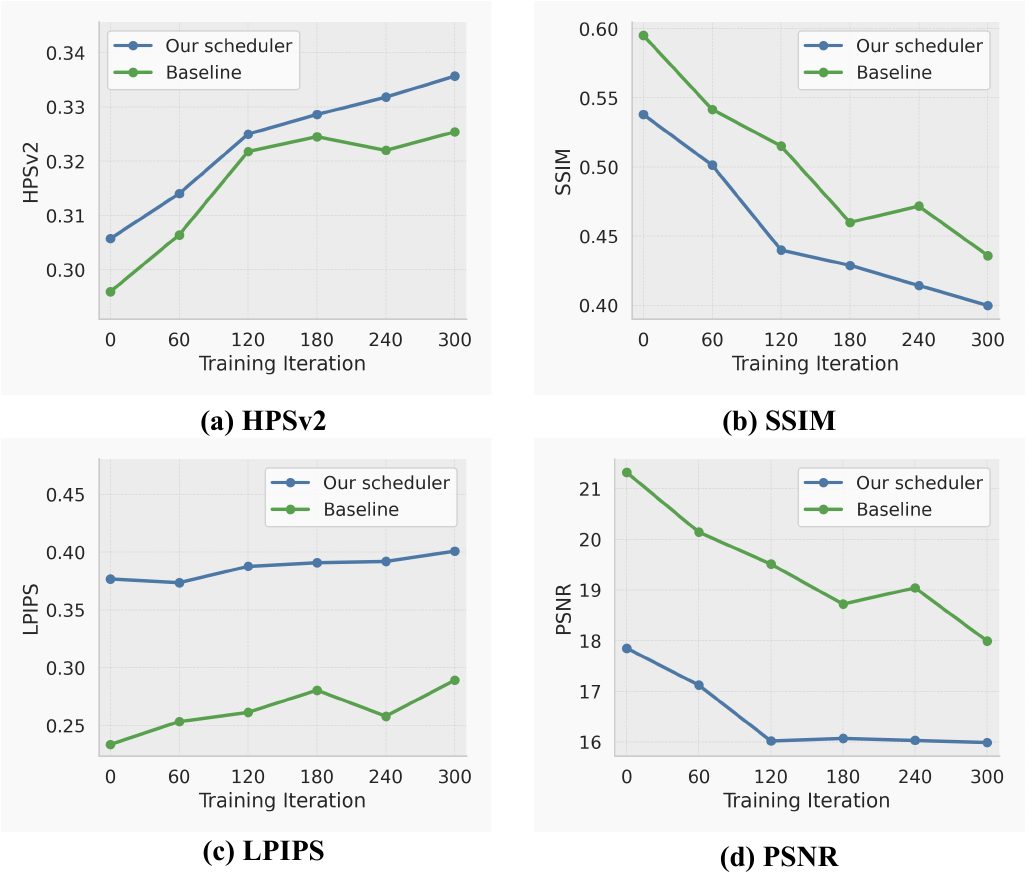}
    \vspace{-7mm}
    \caption{Similarity between samples in diffusion GRPO during the training process. The samples obtained from baseline are highly similar, which reduces the diversity of the samples.}
    \label{fig:sim}
    \vspace{-4mm}
\end{figure}

\noindent\textbf{Stochasticity in Diffusion RL.}
In Fig.~\ref{fig:sim}, we give the curve of similarity between the sampled results using different stochasticity strategies. We employ SSIM, PSNR and LPIPS to measure the similarity between different samples, as well as the result of the trained models.
Note that the lower values of the SSIM and PSNR indicate lower similarity between the sampled results.
Compared to the baseline, the image sampled from our pipeline has lower similarity, which means the diversity of the samples, thus leading to a higher value of HPSv2 score.

%% file: sec/6_Conclusion.tex
\section{Conclusion}
\vspace{-1mm}
\label{Conclusion}
Text-to-image diffusion models often fail to represent all concepts in long, complex prompts. In this study, we propose HiCoGen, which applies Chain of Synthesis to alleviate the difficulty of the model in generating an image from complex text.
It applies an LLM to parse and rewrite the complex prompt into multiple sub-prompts to guide each intermediate image generation. Then, these intermediate images are fed to a subject-driven generative model with the prompt to assemble these concepts until all the concepts in the complex prompt are performed in the final image. To ensure accurate CoS, we introduce a hierarchical reward for diffusion RL fine-tuning, which provides rewards at the global-level, subject-level, and relationship-level. In addition, we prove that optimal sample diversity is achieved by concentrating stochasticity in the early generation stages.
Finally, we propose a new text-to-image benchmark, HiCoPrompt, which challenges text-to-image models to produce images with all the concepts appearing in complex prompts. Experiments on the HiCoPrompt demonstrate the effectiveness of HiCoGen.

%% file: sec/X_suppl.tex
\clearpage
\setcounter{page}{1}

\twocolumn[{%
\renewcommand\twocolumn[1][]{#1}%
\maketitlesupplementary
\begin{center}
    \centering
    \captionsetup{type=figure}
    \includegraphics[width=1\linewidth]{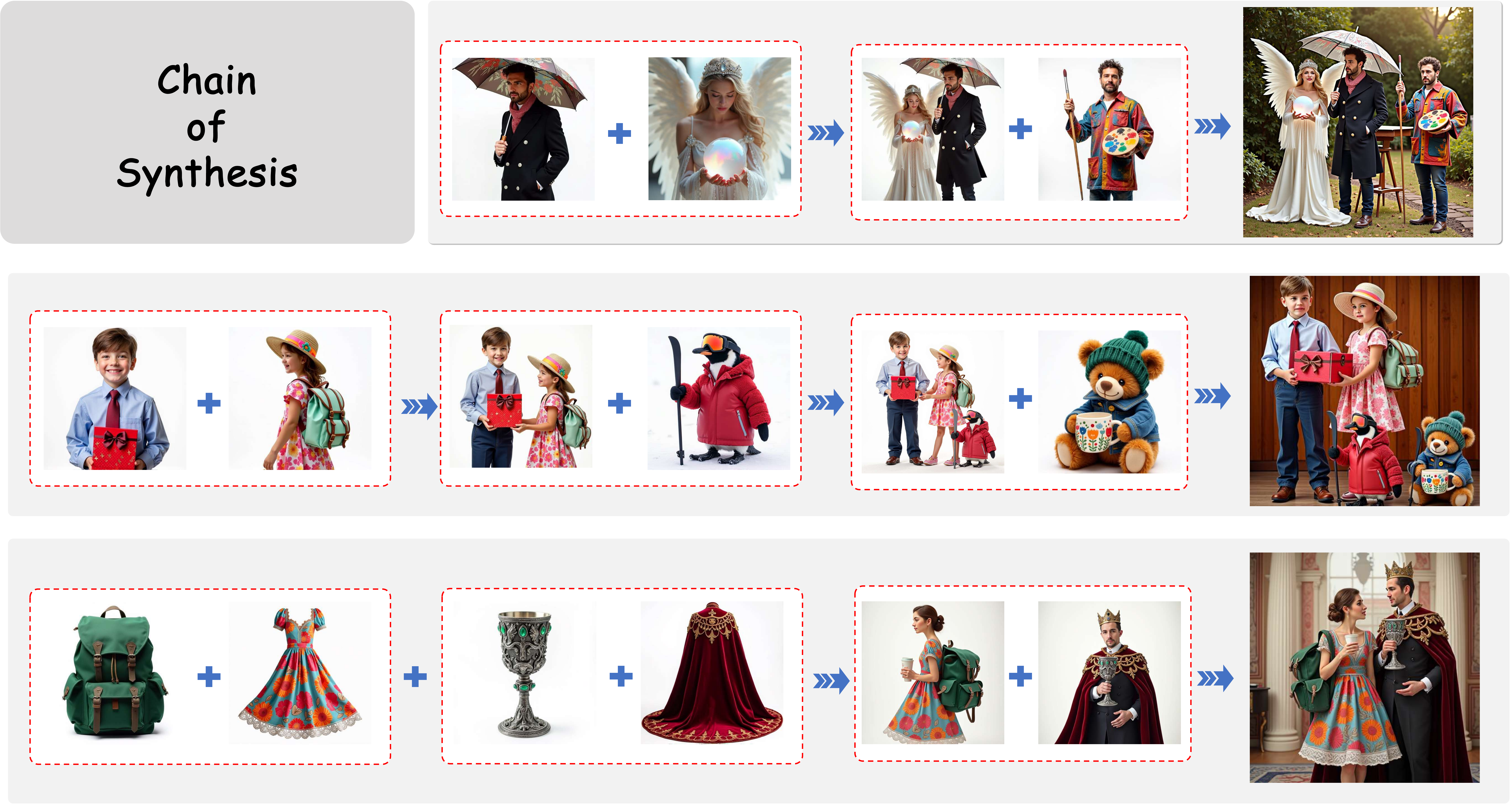}
    \caption{The intermediate results of the Chain of Synthesis.}
    \label{fig:chain_of_synthesis}
\end{center}

\begin{center}
    \centering
    \captionsetup{type=figure}
    \includegraphics[width=1\linewidth]{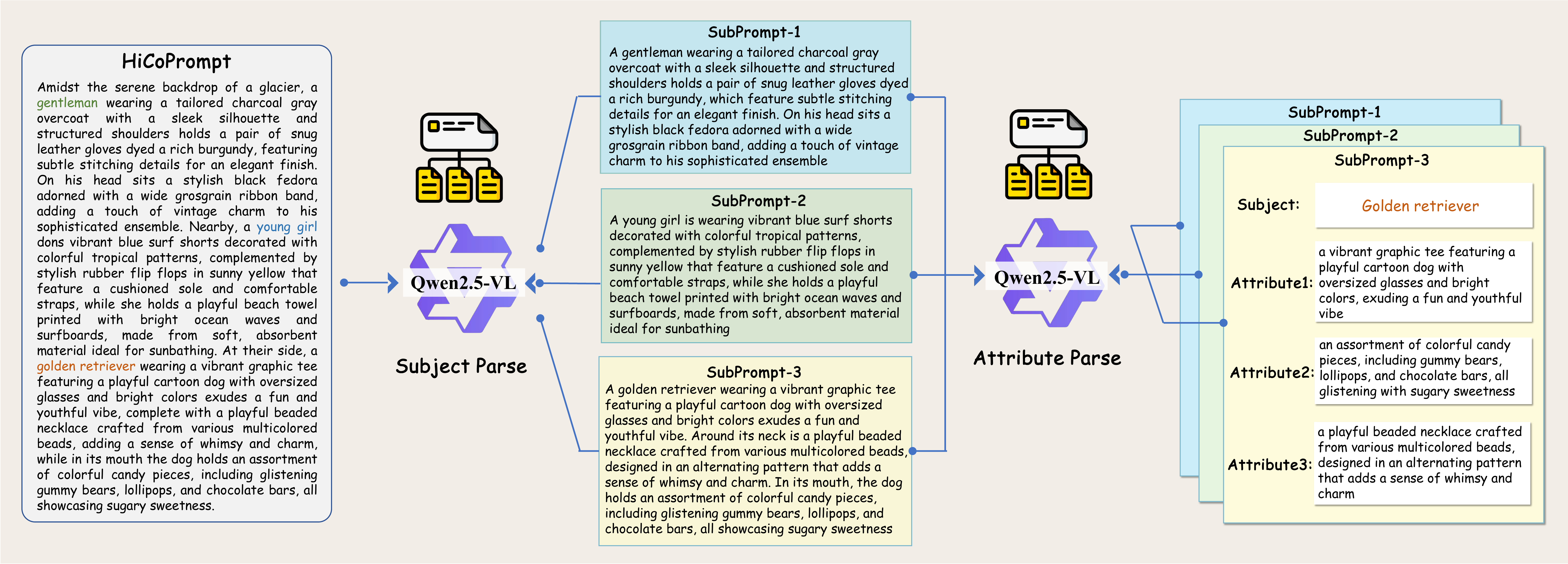}
    \caption{The intermediate prompt of the Chain of Synthesis (zoom in for details).}
    \label{fig:chain_of_synthesis_prompt}
\end{center}

}]

\section{Implementation details}
The weights of the reward are set to $w_\text{clip}=0.7, w_\text{hps}=1.4, w_\text{dino}=0.7, w_\text{vlm}=0.7$.
The size of the trainset and testset is set to 12,000 and 3,000, respectively. The number of de-noise steps is set to 16 when sampling. 
The $\eta_{\max}$ and $\eta_{\min}$ are set to 1.0 and 0.  
The output image size is set to $512\times512$, with the reference image size being $320\times320$. The total training step is set to 300 with a batch size of 12.


\vspace{-2mm}
\section{Chain of Synthesis}
Fig.~\ref{fig:chain_of_synthesis} shows the case of how the HiCoGen performs Chain of Synthesis. The first row showcases the CoS of three subjects in HiCoGen. It first generates two subjects in one intermediate output, and then includes the third subject in the final output. Besides, we also present the scaling results of four subjects in the second row.
The third row shows the result of the HiCoGen on how to generate subjects with hierarchical relationships. For example, it generates clothing and items of the subject, and then generates the final image.
\begin{figure*}
    \vspace{-6mm}
    \begin{subfigure}[b]{\textwidth}
        \centering
        \includegraphics[width=1.0\linewidth]{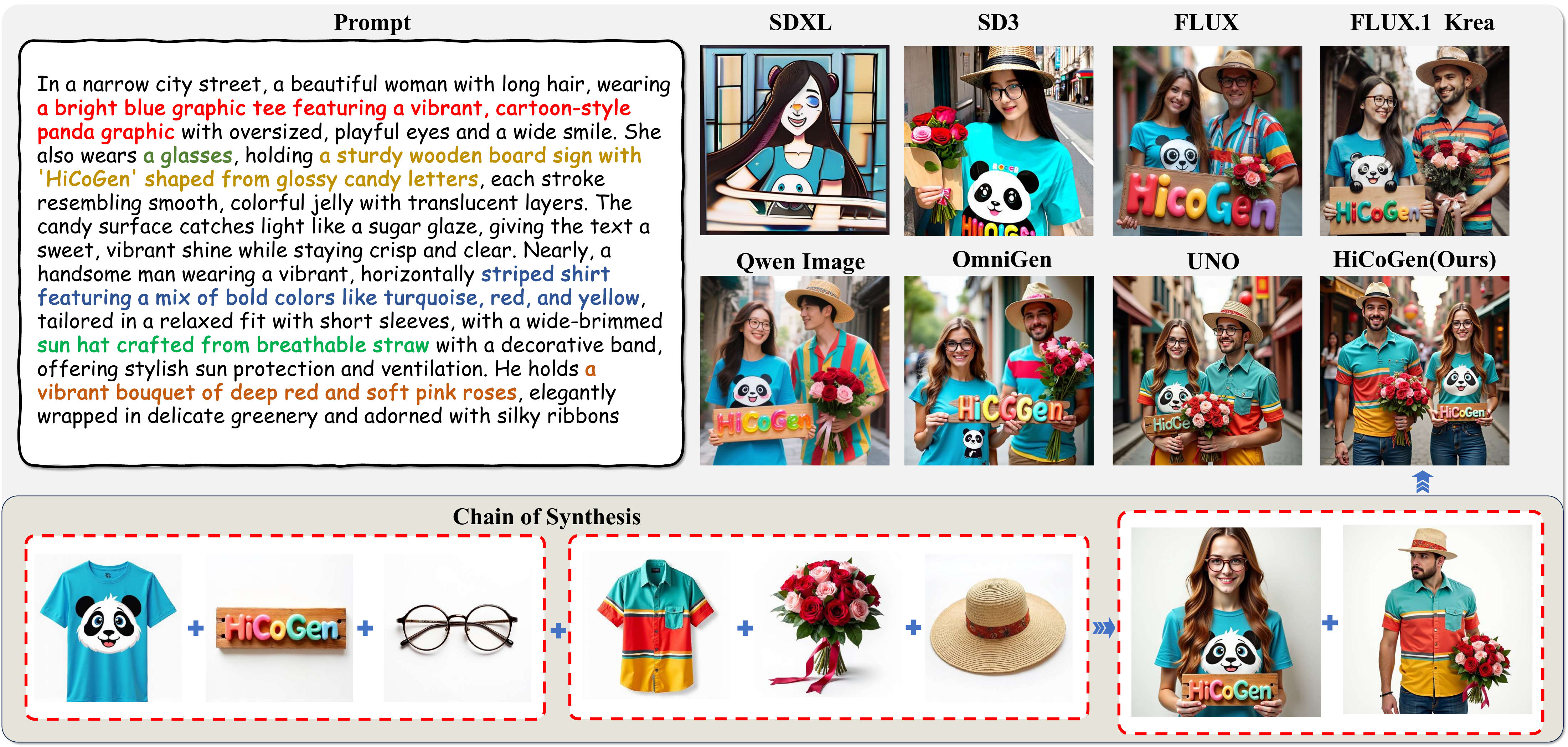}
    \end{subfigure}
    \hfill
    \begin{subfigure}[b]{1\textwidth}
        \centering
        \includegraphics[width=1.0\linewidth]{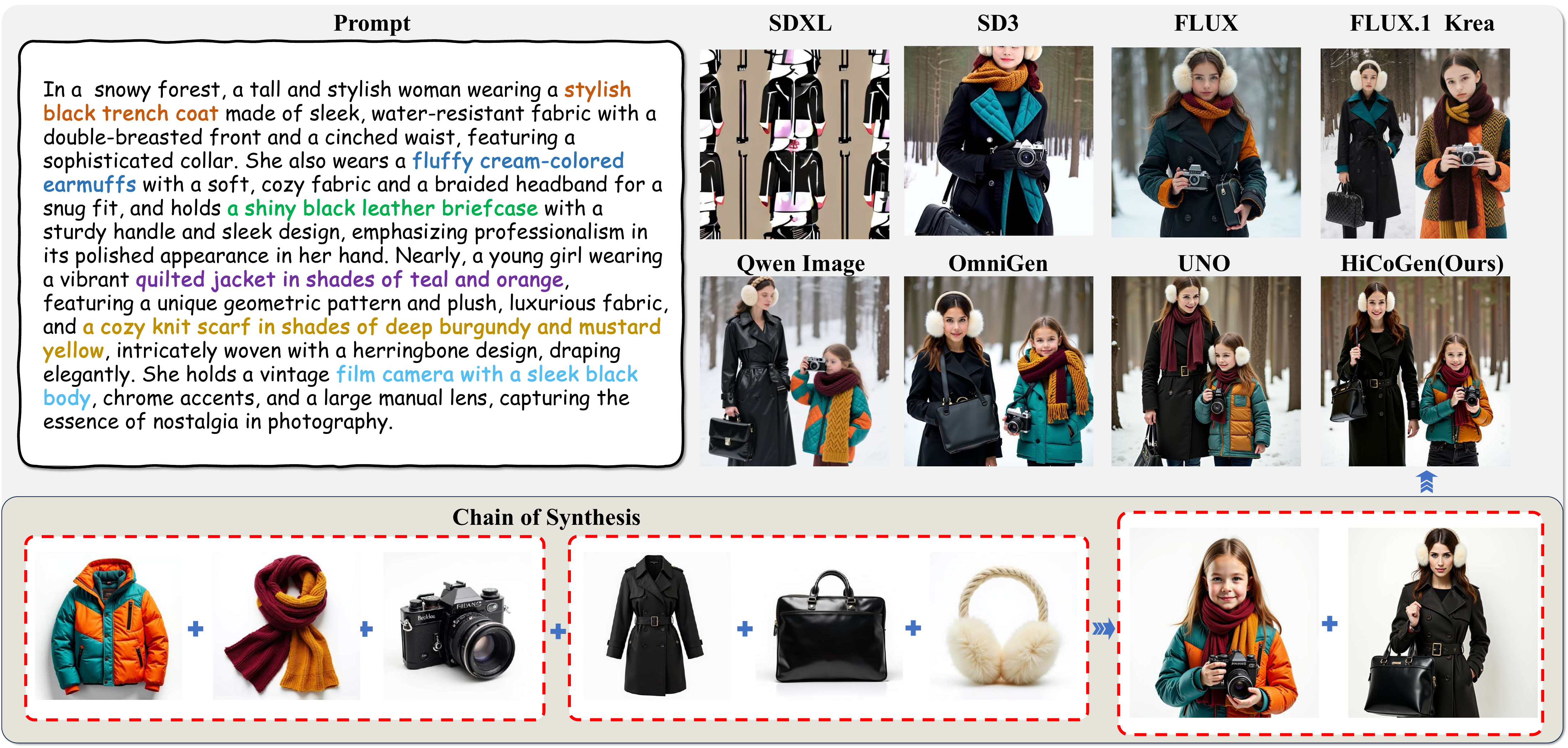}
    \end{subfigure}
    \vspace{-6mm}
    \caption{More visual results of the HiCoGen}
    \label{fig:more_result}
    \vspace{-4mm}
\end{figure*}
Fig.~\ref{fig:chain_of_synthesis_prompt} indicates the whole workflow of the parse and rewrite LLM. For a long and complex HiCoPrompt, it is first parsed 
\newpage
\noindent 
according to subjects (e.g., gentleman, young girl, and golden retriver). 
Then, the sub-prompt is continued to be parsed if it contains multiple attributes that can be generated. Thus, we replace the long and complex prompt with a structured prompt composed of several semantic units.



\vspace{-1mm}
\section{More Visual Results}
Fig.~\ref{fig:more_result} presents additional visual results produced by HiCoGen, as well as the results from other methods. We also illustrate the CoS process involved in the generation. Through CoS, HiCoGen generates all the concepts from the HiCoPrompt.

\clearpage


\begin{figure*}\small

\begin{subfigure}[b]{0.5\textwidth}
\begin{tcolorbox}[
        title={Subject Generation},
        halign=left,
        valign=center,
        nobeforeafter,
        fontupper=\scriptsize
    ]
Role:

Please be very creative and generate 20 groups of components for the given character. 

\vspace{\baselineskip}
Follow these rules:

1. You will be given a $<$character$>$, you need to create its wearing, holding, and accessory items. 

2. These items must exist in the real world (no fantasy or fictional materials). 

3. Do not repeat the same words between outputs; use your imagination and common sense of real life. 

4. Each item must contain **no more than two words**. 

5. Output multiple unique sets if possible.

6. Given a final complete and brief prompt for text-to-image generation

\vspace{\baselineskip}
Output format:

[character]: $<$character name$>$

[clothing1]: $<$real-world clothing$>$

[holding1]: $<$real-world object$>$

[accessory1]: $<$real-world accessory$>$

\vspace{\baselineskip}
Example:

[character]: dog

[clothing1]: superman's costume

[holding1]: sign

[accessory1]: goggles

[brief\_prompt1]: a dog wearing a superman's costume with a goggles, holding a sign

\vspace{\baselineskip}
[clothing2]: space suit

[holding2]: book

[accessory2]: glasses

[brief\_prompt2]: a dog wearing a space suit with glasses, holding a book

 ...
 
 (Up to [asset20])

\vspace{\baselineskip}
 [character]: {character}

\vspace{\baselineskip}

\vspace{\baselineskip}

\vspace{\baselineskip}
\end{tcolorbox}

\end{subfigure}
\hfill
\begin{subfigure}[b]{0.5\textwidth}  
\begin{tcolorbox}[
        title={Attribute Rewrite},
        halign=left,
        valign=center,
        nobeforeafter,
        fontupper=\scriptsize,
    ]
Role:

Please be very creative and generate a detailed prompt for text-to-image generation. 

\vspace{\baselineskip}
Follow these rules:

1. You will be given 3 \{assets\}, you need to create an asset (detailed subject prompt) based on the \{assets\}.

2. These descriptions can refer only to appearance descriptions/or to certain brands. e.g., ``Elon Musk in pajamas'', ``a tiger in a black hat'', ``A Mercedes sports car'', ``A blonde'', ``A rotten wooden door'', ``a book with cover written `Magic'''

3. Do not repeat the same words between outputs; use your imagination and common sense of real life. 

4. Describe this asset in one sentence. No more than 40 words

5. Focus on the asset itself, and must contain the asset in each output.

6. Do not describe its background and environment.

7. You may rewrite it based on its contextual meaning within the original prompt.

8. Do not include any users, wearers, or owners. Describe only the object itself.

\vspace{\baselineskip}
Example:

[input\_asset1]: book

[input\_asset2]: lab coat

[input\_asset3]: necklace

[original\_prompt]: a beautiful woman wearing a lab coat with a necklace and holding a book.

\vspace{\baselineskip}
[output\_asset1]: an open ancient magic book with a thick dark brown tanned leather cover, adorned with hand-embossed golden runes and intricate patterns

[output\_asset2]: a crisp white lab coat with embroidered name and pen-stained pocket

[output\_asset3]: a whimsical necklace with animal-shaped pendants

\vspace{\baselineskip}
Now, please generate the detailed prompt for the following assets:

[input\_asset1]: \{subject1\}

[input\_asset2]: \{subject2\}

[input\_asset3]: \{subject3\}

[original\_prompt]: \{ori\_prompt\}
\end{tcolorbox}
\end{subfigure}

\vspace{\baselineskip}
\begin{subfigure}[b]{\textwidth}  
\begin{tcolorbox}[
        title={Prompt Reframing},
        halign=left,
        valign=center,
        nobeforeafter,
        fontupper=\scriptsize,
    ]
Role:

You are a professional prompt structure optimization and synthesis expert, skilled at merging multiple prompt segments into a single, well-structured and logically coherent prompt.

\vspace{\baselineskip}
Follow these rules:

Generate a unified prompt based on the provided ori\_prompt (original prompt) and subprompts (refined sub-prompts). Please strictly follow these rules:

1. Preserve the internal structure, hierarchy, and semantics of each subprompt.

2. You may adjust sentence structures or logical order slightly during merging to ensure overall fluency and coherence.

3. Do not delete or rewrite any key information from the subprompts.

4. You may adapt the language style according to the tone or purpose of the ori\_prompt (e.g., analysis, creation, reasoning, Q\&A, etc.).

5. The final output should clearly reflect the role and integrated content of each subpart.

6. Output one complete prompt paragraph — no explanations or additional commentary are needed.

7. Do not describe its background and environment.

\vspace{\baselineskip}
[input\_asset1]: \{subject1\}

[input\_asset2]: \{subject2\}

[input\_asset3]: \{subject3\}

[original\_prompt]: \{ori\_prompt\}
\end{tcolorbox}
\end{subfigure}

\end{figure*}

\section{Prompt}

\subsection{Dataset Creation Prompt}
In this section, we demonstrate the prompt used to construct the dataset. First, a subject is randomly selected from a predefined character pool. Then, this subject will be assigned up to three attributes. After that, these attributes are rewritten into a more detailed version. Finally, these prompts are reframed into the final prompts by LLMs.

\clearpage
\subsection{Reward Prompt}
\begin{figure*}[htb]
    
\begin{tcolorbox}[
        title={Subject Reward},
        halign=left,
        valign=center,
        nobeforeafter,
        fontupper=\scriptsize,
    ]
Role

You are an expert AI assistant specializing in the objective evaluation of the consistency of subjects in two images.
Assign a specific integer score to subject. More and larger differences result in a lower score. 

\vspace{\baselineskip}
Important Notes

- Provide quantitative differences in reason whenever possible.

- Ignore differences in the subject's background, environment, position, size, etc.

- Ignore differences in the subject's actions, poses, expressions, viewpoints, additional accessories, etc.

- Ignore the extra accessory of the subject in the second image, such as hat, glasses, etc.

\vspace{\baselineskip}
You must adhere to the output format strictly.

\vspace{\baselineskip}
The score ranges from 0 to 4:

 - 0: No resemblance. The \{subject\} in Image2 does NOT appear in Image1 at all. No matching identifiable features.

 - 1: Minimal resemblance.  The \{subject\} in Image2 appears only minimally in the \{subject\} in Image1. Only trivial or coincidental similarities.

 - 2: Moderate resemblance. The \{subject\} in Image2 partially matches the \{subject\} in Image1. Some major features match, but key identity traits differ.

 - 3: Strong resemblance. The \{subject\} in Image2 is very similar to the \{subject\} in Image1. Most identity traits align, with minor differences.

 - 4: Near-identical. The \{subject\} in Image2 is identical to the \{subject\} in Image1. Nearly identical; same character.

\vspace{\baselineskip}
Output only the brief reason and score. DO NOT OUTPUT any other text. 

\vspace{\baselineskip}
Image1: [brief description of Image1]

Image2: [brief description of Image2]

Reason: [Brief Reason]

Output: [Score]
\end{tcolorbox}
\vspace{\baselineskip}

\begin{tcolorbox}[
        title={Relation Reward},
        halign=left,
        valign=center,
        nobeforeafter,
        fontupper=\scriptsize,
    ]

Role:

You are a visual consistency evaluator. You will compare Image2 and Image1.

\vspace{\baselineskip}
Your task is to check two things:

1) Content Preservation:

   Are the key visual elements from Image2 present in Image1?
   Ignore differences in the subject's background, environment, position, size, etc.

   Examples of key visual elements:
   
   - For a character: clothing type, accessories.
   
   - For clothing: color, pattern, shape, length, structure.
   
   - For a held object: object type, shape

2) Correct Assignment:

   The elements of the character from Image2 must appear on the same character in Image1.
   
   They must not be assigned to another character or misplaced.

\vspace{\baselineskip}
Then classify overall similarity using the scale below:

\vspace{\baselineskip}
0 - Completely mismatched:
    Key elements are missing or assigned to the wrong subject. No meaningful match.

1 - Minimal similarity:
    Only small or generic resemblance. Most defining elements are missing or incorrectly reassigned.

2 - Partial similarity:
    Some important features match, but major elements are missing, altered, or incorrectly assigned.

3 - High similarity:
    Most key features are present and assigned correctly, with only minor inaccuracies.

4 - Fully consistent:
    All key features are preserved and assigned correctly. No confusion or mixing.

\vspace{\baselineskip}
Output only the following format:

\vspace{\baselineskip}
Image1: [brief description of Image1]

Image2: [brief description of Image2]

Reason: [brief explanation of match or mismatch]

Output: [0-4]

\end{tcolorbox}

\end{figure*}

\clearpage
\subsection{Evaluation Prompt}
\begin{figure*}[!htp]
\begin{subfigure}[b]{0.5\textwidth}
\begin{tcolorbox}[
        title={Accuracy of Existing},
        halign=left,
        valign=center,
        nobeforeafter,
        fontupper=\scriptsize,
    ]
Check whether a specific object appears in the image.

\vspace{\baselineskip}

Instruction:

Analyze the provided image carefully. Determine whether the specified item is present or visible in the image.

\vspace{\baselineskip}
Object: \{user\_prompt\}

\vspace{\baselineskip}
Output format:

Answer ``Yes'' if the object is clearly present.
\vspace{\baselineskip}

Answer ``No'' if the object is not visible. Do not assume or infer its presence based on context.

\vspace{\baselineskip}
If uncertain, answer ``Unclear'' and briefly explain why.

\vspace{\baselineskip}
Output only the brief description, reason and answer. DO NOT OUTPUT other text.

\vspace{\baselineskip}
Description: [Brief Description of the Image]

Reason: [Brief Reason]

Output: [Yes / No / Unclear]

\end{tcolorbox}

\end{subfigure}
\hfill
\begin{subfigure}[b]{0.5\textwidth}
\begin{tcolorbox}[
        title={Accuracy of Attribute},
        halign=left,
        valign=center,
        nobeforeafter,
        fontupper=\scriptsize,
    ]

Check whether the object in the image matches the attributes described in the given text.

\vspace{\baselineskip}
Instruction:

Analyze the image carefully. Determine whether the object matches the specific details (such as color, shape, material, pattern, size, part structure, or other explicit attributes) described in the text.
Object description: ``\{user\_prompt\}''

\vspace{\baselineskip}
Output format:

\vspace{\baselineskip}
Answer “Yes” if the object is clearly present and its visible attributes match the given text.

\vspace{\baselineskip}
Answer “No” if the object is present but clearly does not match the described details, or the described object is not visible at all.

\vspace{\baselineskip}
If uncertain, answer “Unclear” and briefly explain why.

\vspace{\baselineskip}
Description: [Brief Description of the Image]

Reason: [Brief Reason]

Output: [Yes / No / Unclear]
\end{tcolorbox}
\end{subfigure}

\vspace{\baselineskip}

\begin{tcolorbox}[
        title={Accuracy of Attribute},
        halign=left,
        valign=center,
        nobeforeafter,
        fontupper=\scriptsize,
    ]

Check whether the specified object in the image interacts with the main subject in a correct and reasonable way according to the following rules.

\vspace{\baselineskip}
Rules:

The object mentioned in the prompt must interact with the main subject in a reasonable, physically consistent way (no floating, no paste-on appearance).

The object must not appear in the hands or possession of another subject.

The object’s size and proportion relative to the main subject must be correct and not distorted.

\vspace{\baselineskip}
Instruction:

Analyze the provided image carefully and judge whether the described object satisfies all the above rules.

Object: ``\{user\_prompt\}''

\vspace{\baselineskip}
Output format:

Output “Yes” if the object is visible and satisfies all three rules.

Output “No” if any rule is violated or the object is not visible.

\vspace{\baselineskip}
Description: [Brief Description of the Image]

Reason: [Brief Reason]

Output: [Yes / No / Unclear]
\end{tcolorbox}

\end{figure*}